%% file: ijcai19.tex
\documentclass{article}
\usepackage{ijcai19}

\usepackage{times}
\usepackage{soul}
\usepackage{url}
\usepackage[hidelinks]{hyperref}
\usepackage[utf8]{inputenc}
\usepackage[small]{caption}
\usepackage{graphicx}
\usepackage{amsmath}
\usepackage{amsfonts}
\usepackage{amssymb}
\usepackage{amsthm}
\usepackage{booktabs}
\usepackage{algorithm}
\usepackage{algorithmic}
\usepackage{cleveref}
\usepackage{bbm}
\usepackage{subcaption}
\usepackage{todonotes}

\newtheorem{theorem}{Theorem}

\newtheorem{defi}{Definition}

\newcommand{\rise}{RELIS}
\newcommand{\nlg}{NLG}
\newcommand{\ctopic}{cross-input}
\newcommand{\intopic}{input-specific}
\newcommand{\NDCG}{\textsc{ndcg}}

\crefformat{section}{\S#2#1#3} % see manual of cleveref, section 8.2.1
\crefformat{subsection}{\S#2#1#3}
\crefformat{subsubsection}{\S#2#1#3}

\usepackage{color}
\newcommand{\ChM}[2][]{\textcolor{black}{#2}}
\usepackage{color}
\usepackage{soul}
\newcommand{\MM}[2][]{\textcolor{black}{#2}}
\newcommand{\MMN}[1]{\textcolor{red}{(\textbf{MM:} #1)}}

\newcommand{\newcite}[1]{\citeauthor{#1} [\citeyear{#1}]}

\title{%Good Rewards Are All You Need: 
Reward Learning for Efficient 
Reinforcement Learning in  \\
Extractive Document Summarisation
}

\author{
Yang Gao$^{1}$\footnote{Work performed while at 
UKP-TUDA.}
\and
Christian M. Meyer$^2$\and
Mohsen Mesgar$^{2}$\And
Iryna Gurevych$^2$
\affiliations
$^1$Dept. of Computer Science, Royal Holloway, University of London\\
$^2$Ubiquitous Knowledge Processing Lab (UKP-TUDA),
Technische Universit{\" a}t Darmstadt\\
\emails
yang.gao@rhul.ac.uk,
\{meyer,mesgar,gurevych\}@ukp.informatik.tu-darmstadt.de
}

\begin{document}

\maketitle

\begin{abstract}
%\emph{Natural Language Generation (NLG)}
\emph{Document summarisation}
can be formulated as a sequential decision-making problem,
which can be solved by \emph{Reinforcement Learning (RL)} algorithms. 
The predominant RL paradigm for summarisation learns a \emph{cross-input}
policy, which requires considerable time, data and parameter tuning
due to the huge search spaces and the delayed rewards.
Learning \emph{input-specific} RL policies is a
more efficient alternative\MM[,]{} but so far depends on handcrafted rewards,
which are difficult to design and yield poor performance.
We propose \rise{}, a novel RL paradigm
that learns a reward function
with \emph{Learning-to-Rank (L2R)} algorithms 
at training time and uses this reward function to train an 
\mbox{input-specific} RL policy at test time.
We prove that \rise{} guarantees to 
generate \mbox{near-optimal} summaries 
with appropriate L2R and RL algorithms.
Empirically, we evaluate our approach on 
extractive multi-document summarisation.
%\emph{extractive \mbox{multi-document} summarisation}.
We show that \rise{} reduces the training 
time by two orders of magnitude compared to 
the state-of-the-art models while performing on par with them.   
%}

\end{abstract}

\input{intro.tex}
\input{relish.tex}

\input{proof.tex}

\input{experiments.tex}

\input{results.tex}

\input{related_work.tex}

\section{Conclusion}
\label{sec:conclusion}
We propose a novel RL paradigm %for \nlg{}
called \rise{},
which learns a reward function from a reward oracle
with \MM{\mbox{learning-to-rank} (L2R)} algorithms at training time\MM[and]{, and then} 
 uses the learnt reward to train
\intopic{} RL policies at test time.
Compared with the widely \MM[used]{employed} 
\ctopic{} RL-based summarisation approaches,
\rise{} avoids the expensive learning of \mbox{\ctopic{}} policies
but, instead, \MM{efficiently performs}\MM[relatively cheap]{} L2R and \intopic{} RL learning.
\MM[Also]{Moreover}, \rise{} avoids the 
\MM[difficult]{arduous} reward design required
in \intopic{} \mbox{RL-based} summarisation approaches.
We prove that, 
%for NLG tasks that can be formulated
%as discrete optimisation
with proper L2R and RL algorithms,
\rise{} guarantees to produce near-optimal outputs.
\ChM[Experiments on EMDS show that, even using standard
reward features and linear L2R and RL algorithms, 
\rise{} yields performance on par 
with the state-of-the-art neural-based methods while
only requiring a small fraction of data or time to train.]{Our experiments show that even with linear L2R and standard RL algorithms, \rise{} yields performance on par 
with the state of the art while
only requiring a small fraction of data and time to train.}
% % MM: I suggest to remove it from conclusion 
%
% The strong performance of \rise{} sheds light on the importance
% of \emph{reward learning} for RL, %-based \nlg{} approaches,
% a problem \ChM{that} recently receives interest from the 
% machine learning community 
% \cite{DBLP:conf/nips/IbarzLPILA18} %,DBLP:conf/nips/ZhengOS18} 
% but is largely overlooked in \nlg{} until now.
% %
% 
% Unlike classic RL applications, \MM[e.g.,]{e.g.}\ robotics and games,
% where rewards are provided or %mostly obvious and hence 
% easy to design,  
% \nlg{} tasks lack strong metrics \MM[for measuring]{to measure}
% the quality of the output.
% %
% \MM[Widely used]{Well-known} rewards such as BLEU and ROUGE have \MM[long]
% been criticised for their low correlation
% with human evaluations \cite{chaganty2018price}. 
% %
% Recent work \MM[also] suggests that \MM[a small improvement of]{improving}
% the reward function \MM[can] significantly \MM[boost]{boosts} the performance
% of RL-based NLG \cite{DBLP:conf/emnlp/KryscinskiPXS18}.
% %
%
Our work lays the theoretical foundation for reward learning
in  NLG, and we hope it will encourage further
research in this direction.%on reward learning 
%for other RL-based NLG tasks.
%\MM[, and]{} 
%We hope that it encourages further research
%in this area.

\section*{Acknowledgements}
%The authors thank the anonymous reviewers for their
%constructive comments, and thank Steffen Eger for his helpful
%remarks on the proof in this work.
%
This work has been supported by the German
Research Foundation (DFG), as part of the 
QA-EduInf project (%grant 
GU 798/18-1 and %grant 
RI 803/12-1) and through the 
German-Israeli Project Cooperation 
(DIP, %grant 
DA 1600/1-1 and %grant 
GU 798/17-1).

%\clearpage

\bibliographystyle{named}
\bibliography{general_long}

\end{document}

%% file: intro.tex
\section{Introduction}
\label{sec:intro}
\emph{Extractive document summarization}, as a challenging 
instance of \emph{natural language generation} (NLG), 
is a popular summarisation paradigm, 
which builds summaries by selecting
an appropriate sequence of important 
phrases or sentences from the input document(s).
%
% \emph{Document summarisation} is a crucial task in
% \emph{Natural Language Generation (\nlg{})}.
% %
% \emph{Extractive summarisation} is a popular summarisation 
% paradigm, which builds summaries by selecting
% an appropriate sequence of phrases or sentences from the 
% input document(s).
%
%Given an input document or cluster of documents,
%the task of summarisation is to select an appropriate sequence of 
%words, phrases or sentences so as to 
%build a summary complying with certain requirements (e.g.
%length restriction). 
%
Extractive summarisation can be formulated as 
a sequential decision-making problem,
and hence can be tackled by the 
\emph{Reinforcement Learning (RL)} algorithms. 
RL searches for the (near-)optimal \emph{trajectories}
(i.e.\ sequences of decisions) by 
directly optimising the objective functions,
e.g. the ROUGE metrics \cite{lin2004rouge}. 
Such objectives are non-differentiable and therefore 
difficult to be directly optimised by deep neural networks.
In addition, RL alleviates the \emph{exposure bias} 
problem faced by sequential supervised learning paradigms in \nlg{}.
Combining RL algorithms such as REINFORCE \cite{DBLP:journals/ml/Williams92}
with neural techniques (e.g.\ sequence-to-sequence)
yields \mbox{state-of-the-art} performance in 
summarisation \cite{DBLP:journals/corr/abs-1802-08636,DBLP:journals/ijon/YaoZLW18}.

\begin{figure*}
    \centering
    \includegraphics[width=.95\textwidth]{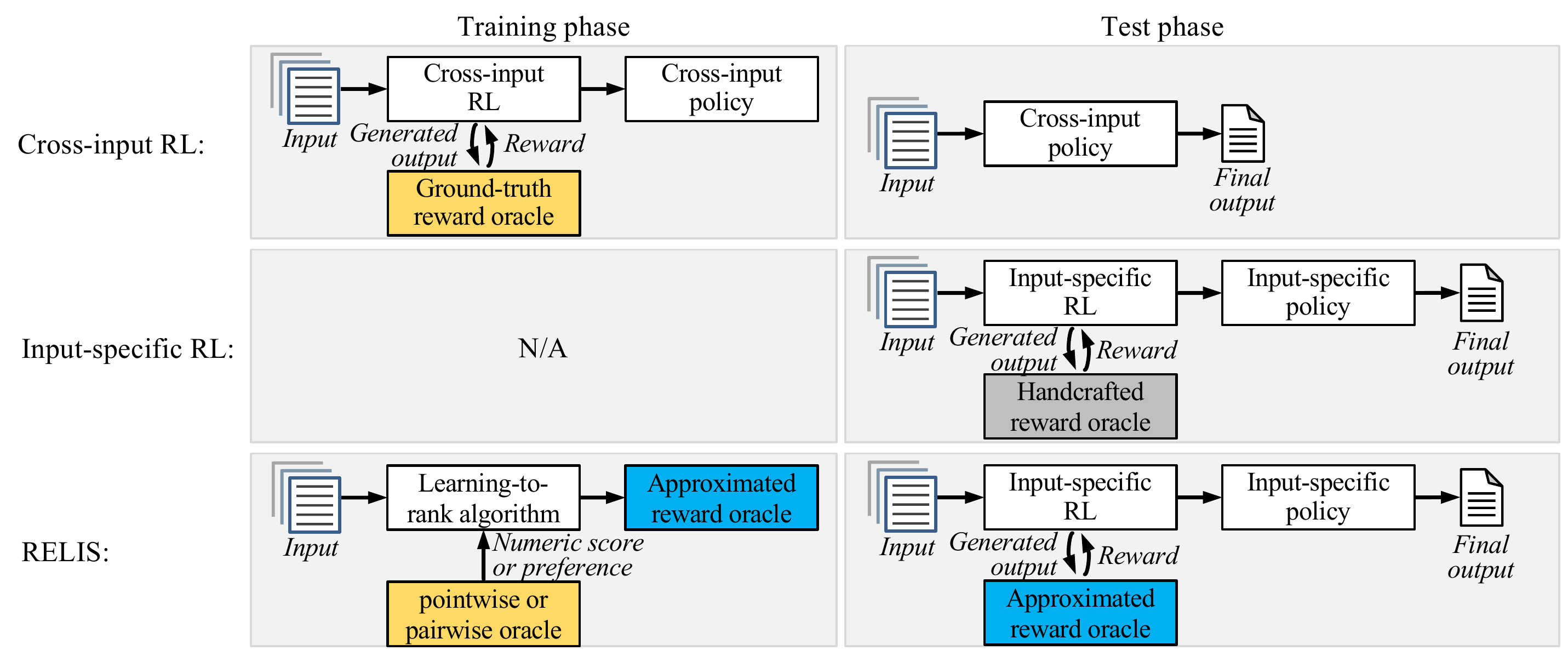}
    \caption{\MM[W]{The w}orkflows of \ctopic{} RL (top),
    \intopic{} RL (\MM[centre]{middle}) and \rise{} (bottom).
    The ground-truth reward can by 
    \MM[either provided]{provided} by humans or 
    automatic metrics (e.g.\ BLEU or ROUGE) 
    measuring the similarity between generated output text and the reference output.}
    \label{fig:workflows}
\end{figure*}

%\todo{if we keep RELIS, change fig1! change also gold standard to ground truth}
%\todo{change gold standard to ground truth in the text}
    
%cross topic RL
Existing RL-based summarisation systems fall into two
categories: \emph{\ctopic{} RL} and \emph{\intopic{}} RL.
For \mbox{\ctopic{}} RL
(upper part in Fig.~\ref{fig:workflows}),
at training time, RL agents interact with a
\emph{\ChM[gold-standard]{ground-truth} reward oracle} for multiple \emph{episodes}
so as to learn a \emph{policy} that maximises the
accumulated rewards in the episode;
at test time, the learnt policy is applied to
unseen data to generate the summary. 
However, learning such \mbox{\ctopic{}} policies
is very expensive because of the huge search space. 
Another issue is the \emph{delayed rewards}:
the ROUGE scores can be calculated only \MM[after]{when} 
the complete summary is generated.
Such delayed rewards cause RL-based summarisers 
to take even longer time to converge.
Although multiple techniques have been proposed
to speed up the training of RL, e.g.\ memory replay 
\cite{mnih2015human},
MIXER \cite{DBLP:journals/corr/RanzatoCAZ15} and 
imitation learning \cite{DBLP:conf/uai/ChengYWB18},
training \ctopic{} RL 
yet requires considerable time, data and parameter tuning.
%\cite{rlblogpost}.

%in topic RL
On the other hand, \intopic{} RL 
(middle part in Fig.~\ref{fig:workflows})
neither requires parallel data (i.e.
input documents and reference summaries for them) 
nor a reward oracle.
Instead, it employs a handcrafted reward function 
with which the RL agent interacts in order to learn a policy specifically
for the given input.
By doing so, the size of the search space significantly 
decreases, which consequently diminishes the training time and computational resources.  
However, designing such a reward function is highly challenging 
as it should fit all inputs \cite{rioux2014emnlp}.
This explains the poor performance of \intopic{} RL 
summarisers \cite{ryang2012emnlp}.

To tackle the problems faced by
the above two RL-based summarisation paradigms, 
%in this work, 
we propose a novel %\YGNew[RL-based summarisation]{} 
paradigm called 
\emph{REward Learning for Input-Specific reinforcement learning (\rise).}
Instead of learning a \ctopic{} policy, \rise{} learns
a \emph{\mbox{\ctopic{}} reward oracle} at training time,
\MM[and]{and then} uses the learnt reward to train 
an \intopic{} policy for each input at test time 
(bottom part in Fig.~\ref{fig:workflows}).
\rise{} is inspired by \emph{inverse RL}
\cite{abbeel_ng2004app_learning}, which requires
a demonstrator to present optimal trajectories. 
Because such a demonstrator is hardly available
in summarisation, \rise{} leverages \emph{Learning-to-Rank (L2R)}
algorithms \cite{DBLP:journals/ieicet/Li11}
to approximate the \ChM[gold-standard]{ground-truth} reward oracle
from ``weak supervisions'', \MM[e.g.]{e.g.,}
numeric scores that indicate the quality of the 
summary or preferences over summary pairs.
%
%L2R ensures high ranking correlation between the learnt oracle
%and the gold-standard oracle by %
%querying the gold-standard oracle \YG{for ``weak supervisions'',
%e.g.} 
%\MMN{The above sentence is still too long and hard to read. Any idea?
%-YG: does the above small change make it better?}

Our contributions are threefold: 
\textbf{(i)}
We propose RELIS (\S \ref{sec:background}), 
a new RL-based summarisation framework
that enjoys the strong performance of \ctopic{} RL
and the low computational cost of \intopic{} RL.
\textbf{(ii)}
Theoretically, we prove that
by employing appropriate L2R and RL algorithms, 
\rise{} guarantees to generate near-optimal outputs (\S \ref{sec:proof}).
\textbf{(iii)} 
Empirically, we evaluate \rise{} in
multi-document extractive summarisation 
(\S \ref{sec:experiments}). 
Compared to the state-of-the-art methods, 
\rise{} provides comparable %or significantly better
performance %in terms of both ROUGE and human evaluations,
but requires much less the training time or data. 
Because the proof of \rise{} is generic, 
we believe \rise{} has the potential to be
applied to other \nlg{} tasks,
e.g.\ translation and sentence simplification.
%
%Fast training is particularly desirable in 
%\nlg{} tasks where new data arrive on the fly, 
%for example in real-time summarisation \cite{DBLP:conf/trec/LinMSTGAMMV17}.
Source code and supplementary material are available at 
\url{https://github.com/UKPLab/ijcai2019-relis}.

%% file: relish.tex
% \ChMN{move related work here and rename background to `model' or `approach'.
% YG: not sure if we should move related work to here; the current RW
% is more suited to the end position; if we move RW to here, we
% should add some 'background' material, maybe from MM's version. 
% -MM: i think introduction is suffient for this.}
\section{\rise{}}
\label{sec:background}

We first formally define the 
summarisaiton task,
and then detail the L2R and RL module of \rise{}.

\iffalse
\ChM[%
We first formally define \nlg{} tasks
(\S\ref{subsec:back:emds}), and then
introduce the L2R technique  
(\S\ref{subsec:back:pref_learn}), which we use
to approximate the reward oracle.
Finally, we describe how we use RL to solve the MDP formulation of NLG tasks (\S\ref{subsec:back:rl}). 
%\MMN{Do we really define RL? It's unclear to me in this place. -- decision: we formulate NLG as MDP}.
]{%
We first formally define \nlg{} tasks
(\S\ref{subsec:back:emds}) and 
introduce L2R for approximating the reward oracle (\S\ref{subsec:back:pref_learn}).
Then, we describe our usage of RL for \nlg{} (\S\ref{subsec:back:rl}). 
}
\fi

\subsection{Extractive Summarisation}
%\label{subsec:back:emds}
In line with \newcite{DBLP:conf/acl/PeyrardE17a}, 
we formulate summarisation as a 
\emph{discrete optimisation problem}. 
Let $\mathcal{X}$ be the set of all possible input documents
and $\overline{\mathcal{X}}$ be the set of 
inputs available at training time.
An input can be either a single document or a cluster
of documents on the same topic.
For input $x \in \mathcal{X}$, let $\mathcal{Y}_x$ 
indicate the set of all extractive summaries for $x$
that comply with the length requirement. 
%
%Note that $\mathcal{Y}_x$ can include either extractive
%or abstractive summaries.
%
Then, the task of summarisation is to map each input $x$ 
to its best summary in $\mathcal{Y}_x$ with respect to
a \emph{ranking function} 
$\sigma_x\colon \mathcal{Y}_x \rightarrow \mathbb{N}$.
For a candidate summary $y\in\mathcal{Y}_x$, $\sigma_x$ 
returns the number of candidates 
in $\mathcal{Y}_x$ that have
equal or lower quality than $y$, including $y$ itself.
For example, if $y$ is the highest-quality candidate 
in $\mathcal{Y}_x$, then $\sigma_x(y) = |\mathcal{Y}_x|$.
%\footnote{In this
%paper, we denote the cardinality of a set $A$ by $|A|$.}
%
%\MM{The ranking function }
$\sigma_x$ can be obtained from human evaluators,
automatic %\MM{evaluation} 
metrics (e.g.\ ROUGE) or heuristics measuring
the quality of outputs. We denote the ground-truth ranking
function on $\mathcal{Y}_x$ by $\sigma^*_x$. 
%\MMN{should we refer to the notation table in appendix here?}

Given the above definition of summarisation,
we define a summariser agent as a tuple $(\sigma_x,M_x)$,
where $M_x$ is an \emph{optimisation model}
for finding the (near-)optimal summary
with respect to $\sigma_x$.
In the \ctopic{} paradigm, the RL agent 
learns a policy that solves the optimisation 
problem $M_x$ for any $x \in \mathcal{X}$ at training time.
In \rise{}, instead, an agent learns a ranking 
$\Hat{\sigma}_x^U\colon \mathcal{Y}_x \rightarrow \mathbb{N}$ 
at training time so that 
$\Hat{\sigma}_x^U$ is as ``close'' as possible to
$\sigma^*_x$ (``close'' will be formally defined 
in \S\ref{sec:proof}).
%; the superscript $U$ in 
%$\Hat{\sigma}_x^U$ will be described in the next paragraph).
At test time,
for each $x^{\prime} \in \mathcal{X} \setminus \overline{\mathcal{X}}$,
\rise{} formulates $M_{x^\prime}$ as a
\emph{Markov Decision Process (MDP)} and learns an 
RL policy specifically for 
%the optimisation problem 
$M_{x^{\prime}}$.

\subsection{Learning to Rank (L2R)}
%\label{subsec:back:pref_learn}
L2R algorithms learn to reconstruct the ranking over objects from an \emph{oracle} \cite{DBLP:journals/ieicet/Li11}.
%For $x \in \overline{\mathcal{X}}$
%\ChM[%
%L2R learns a \emph{utility function}
%$U$ from the oracle, which is then used to induce the approximated
%ranking $\Hat{\sigma}_x$, such that $\Hat{\sigma}_x(y_1) > \Hat{\sigma}_x(y_2)$
%if and only if $U(y_1,x) > U(y_2,x)$.
L2R induces the approximated
ranking $\Hat{\sigma}_x^U$ by learning a \emph{utility function} $U$ from the oracle, such that $\Hat{\sigma}_x^U(y_1) > \Hat{\sigma}_x^U(y_2)$
if and only if $U(y_1, x) > U(y_2, x)$.
L2R can learn from different \MM{types of} oracles, including \emph{pointwise} oracles that provide \mbox{point-based} scores
for objects, % (e.g., on a Likert scale), 
\emph{pairwise} oracles that provide preferences
over pairs of objects, and \emph{listwise} oracles that
provide certain metrics (e.g.\ accuracy) for candidate rankings.
%\todo{YG: I change all 'which' back to 'that', because we follow the British English usage throughout the paper, including spelling and sub-clause which/that,}
%\MMN{ https://en.oxforddictionaries.com/usage/that-or-which}
%
\MM[In this work]{Here}, we focus on pointwise and pairwise oracles \MM[, 
as]{as} humans \MM[can reliably]{reliably} provide such judgements for short texts  
\cite{DBLP:conf/acl/RiezlerKU18}.
\MM[We can approximate $U$]{Function $U$ can be learnt by} any
function approximation \MM[techniques]{technique}, e.g., neural networks. 
%or linear regression.
%\todo{\MMN{not clear how the function $U$ is defined! it's left as a freedom for the model, right? should we say it here? We also need to say that U can be taken a representation for reward because section 4.2 is titled Reward Representation} YG adds a sentence} 

In pointwise L2R, for every $x \in \overline{\mathcal{X}}$, we draw
$N$ sample summaries from $\mathcal{Y}_x$ using some sampling strategy,
e.g., random sampling without replacement.
Then, we query their $\sigma^*_x$ values from the oracle
and use a regression algorithm to 
minimise the averaged mean squared error (MSE)
between $U$ and $\sigma^*_x$\MM[on 
the sampled $N$ outputs $y_i \in \mathcal{Y}_x$.]{.}
Formally, the loss function is
\begin{align}
\mathcal{L}^\mathrm{MSE} = \frac{1}{N \cdot |\overline{\mathcal{X}}|}
\sum_{x \in \overline{\mathcal{X}}} \:\sum_{i=1}^{N} \: 
(\sigma^*_x(y_i)-U(y_i,x))^2.
\label{eq:mseloss}
\end{align}

In pairwise L2R, for every $x \in \overline{\mathcal{X}}$
we sample $K$ pairs of summaries from $\mathcal{Y}_x$
and then query the oracle about their preferences. 
%\todo{\MMN{in practice we draw M samples from these N outputs, right? not sure if we should bring it here!}}
We denote the collected preferences 
\ChM[among the $N$ sampled candidate pairs for]{for} input $x$ by 
\mbox{$P_x = \{p_1(y_{1,1},y_{1,2}), \cdots, p_K(y_{K,1},y_{K,2})\}$}, 
where $y_{i,1}, y_{i,2} \in \mathcal{Y}_x$ 
are the candidate outputs presented to the oracle in the 
$i^\mathrm{th}$ iteration;
$p_i(y_{i,1},y_{i,2})$ equals $1$ 
if the oracle prefers $y_{i,1}$ over $y_{i,2}$,
and equals $0$ otherwise.
Different loss functions can be used to learn $U$ 
from the preferences in $P_x$.
First, we consider the cross-entropy loss\MM[, also known as the Bradley-Terry model]: %\cite{bradley1952rank}:
%\todo{footnote takes quite some space, better integrate in text? Same for hinge loss}
%
\begin{align}
\mathcal{L}^\mathrm{CE} = - &\:
      \sum_{x \in \overline{\mathcal{X}}} \:
      \sum_{p_i(y_{i,1}, y_{i,2}) \in P_x} \!\!\!\!\!  
      \quad[\;  p_i(y_{i,1}, y_{i,2})\,\log \mathcal{P}(y_{i,1}, y_{i,2}) \nonumber \\ 
      & +\;   p_i(y_{i,2}, y_{i,1})\,\log \mathcal{P}(y_{i,2}, y_{i,1}) \, ],
      \label{eq:ce_loss}
\end{align}
where $\mathcal{P}(y_1, y_2) = (1+\exp[U(y_2,x)-U(y_1,x)])^{-1}$.
\ChM[Another widely used loss function for pairwise L2R]{An 
alternative}
is the margin ranking \ChM[loss, also known as the 
pairwise hinge loss:]{(a.k.a.\ pairwise hinge) loss:}
\begin{align}
\mathcal{L}^\mathrm{MR} = & 
\frac{1}{N \cdot |\overline{\mathcal{X}}|} 
\sum_{x \in \overline{\mathcal{X}}} \: \sum_{p_i(y_{i,1}, y_{i,2}) \in P_x}  
\max[0, \nonumber \\ 
& 1-e_i \cdot (U(y_{i,1},x) - U(y_{i,2},x))],
\label{eq:hinge_loss}
\end{align}
where $e_i = 1$ if $y_{i,1}$ is preferred over $y_{i,2}$,
and $e_i = -1$ otherwise. %\ChMN{to avoid any confusion with document inputs $d$, shall we use a different symbol for $d_i$ here?}. 
Additionally, we consider an improved margin ranking loss
proposed by \newcite{DBLP:conf/ecir/AgarwalC10},
%\MM[which gives larger penalties to ``wider-margin'' misranks:]{which penalises incorrect ranks more than correct ones:}
which gives large penalties to \ChM{mis-ranked} pairs with wide margins\ChM[
but are mis-ranked:]{:}
%\todo{YG: all loss functions penalise incorrect ranks more than
%correct ranks. Improved margin ranking loss gives larger penalty
%if two items have large utility gaps but the ranker
%still mis-rank the pair.}
\begin{align}
\mathcal{L}^\mathrm{IM} = & \frac{1}{N \cdot |\overline{\mathcal{X}}|}
\sum_{x \in \overline{\mathcal{X}}} \: \sum_{p_i(y_{i,1}, y_{i,2}) \in P_x} 
\max[0, \nonumber \\
& \hspace*{-.5cm}|\sigma^*_x(y_{i,1}) - \sigma^*_x(y_{i,2})|  
- e_i (U(y_{i,1},x) \!\! - \!\! U(y_{i,2},x))].
\label{eq:weighed_hinge_loss}
\end{align}
Note that the improved margin ranking loss is a mix of
pointwise and pairwise L2R \MM[, because]{as} it requires
both \ChM[the $\sigma^*_{x}$ values and the preferences over
the sampled pairs of outputs]{$\sigma^*_{x}$ and $P_x$}.%
%When we want to highlight that the approximated ranking 
%is derived from $U$, we denote the approximated
%ranking by $\Hat{\sigma}^U_D$.
%\ChMN{the last sentence does not really fit here; maybe we can put it to a later position in the paper}

%\subsection{Reinforcement Learning \ChM{(RL)}}
\subsection{Reinforcement Learning (RL)}
%\label{subsec:back:rl}
RL amounts to algorithms 
that search for optimal solutions in 
\emph{Markov Decision Processes (MDPs)}.
We formulate the optimisation problems $M_x$ for \nlg{}
tasks as episodic MDP \cite{ryang2012emnlp,rioux2014emnlp}.
An episodic MDP is a tuple $(S,A,P,R,T)$, where $S$ is \ChM[the]{a} set of 
\emph{states}, $A$ is \ChM[the]{a} set of \emph{actions}, 
$P\colon S \times A \to S$ 
is the \emph{transition function} with $P(s,a)$ giving
the next state after performing action $a$ in state $s$\MM[.]{, }
$R\colon S \times A \to \mathbb{R}$
is the \emph{reward function} with $R(s,a)$ giving the immediate reward for
performing action $a$ in state $s$ and
$T \subseteq S$ is \ChM[the]{a} set of 
\emph{terminal states} that mark the end of an episode. 
%As such, generating a text is an episode in episodic MDPs. 
%
Furthermore, to formulate the delayed rewards in summarisation,
we let $R(s,a) = 0$ if $P(s,a) \not \in T$,
so as to ensure that non-zero rewards only appear 
in the last step of each episode.
%\ChM{\ (sometimes called \emph{delayed rewards})}.
%\MMN{is it the reward of the whole sequence or just the a state before the terminal state? --\textbf{ChM}: all rewards are zero, except if the chosen action ends in a terminal state (which can be either an action yielding an overlength summary or the terminate action ending in the the absorbing state $s_T$).}.

In extractive summarisation, the components in 
$M_x$ are defined as follows.
%
%For an input document (or document cluster) $x \in \mathcal{X}$, 
$S$ is the set of all possible states, i.e.\ permutations
of sentences from $x$. 
Note that $\mathcal{Y}_x$ is a subset of $S$ because
$\mathcal{Y}_x$ only includes the summaries complying with the length requirement, but $S$ includes all possible extractive summaries. 
%\MMN{totally confused with notations, that's why you have note here! however, it was difficult to understand why $s_d$ is subset of S until I read about T includes .... }
\ChM[$A$ consists of two action categories]{Two types of action constitute $A$}: \emph{add} a new sentence from the input document cluster $x$ to the current draft summary \MM[or]{and} \emph{terminate}
the generation. $T$ includes all the over-length
summaries and an \emph{absorbing state} \MM[$s_T$,]{$s_T$.}
\MM[such that if]{If} action $a$ is terminate,
$P(s,a) = s_T$ regardless of the current state $s$.
The reward function $R(s,a)$ returns $\sigma_x(s)$ if 
action $a$ is terminate, 
a negative reward if the current state $s$ is 
an over-length summary and 0 otherwise.
We denote this MDP by $M_{x}(\sigma_x)$ to highlight that
it uses $\sigma_x$ as rewards.

A \emph{policy} $\pi\colon S \times A \to \ChM[\mathbb{R}]{[0,1]}$ 
defines how actions are selected: 
$\pi(s,a)$ is the probability
of selecting action $a$ in state $s$.
In the context of summarisation,
%(see \S\ref{subsec:back:emds}), 
we slightly abuse the notation by letting
$\pi(y|x)$ denote the probability of policy $\pi$
generating $y \in \mathcal{Y}_x$ given input $x$.
\MM[and w]{W}e say a policy $\pi$ is \emph{proper} if
$\sum_{y \in \mathcal{Y}_{x}} \pi(y|x) = 1$, i.e.\ $\pi$ does not 
generate illegal outputs.
Then, the expected reward of performing proper policy $\pi$ is:
\begin{align}
\mathcal{R}_{\sigma_x}(\pi|x)  = \mathbb{E}_{y \sim \pi}
\sigma_x(y) = \sum\limits_{y \in \mathcal{Y}_x} \pi(y|x)\sigma_x(y).
\label{eq:rl_obj}
\end{align}%\ChMN{should the $y$ be a $g$ in the previous equation?}
The goal of RL is to find the optimal policy $\pi^*$
that has the highest expected reward: 
$\pi^* =\arg\max_{\pi}\mathcal{R}_{\sigma_x}(\pi|x)$.
%Because $\sigma_x$ is a ranking function, 
Note that $\pi^*_x$ is a probability
distribution that assigns non-zero probabilities only
to the optimal outputs for $x$, i.e.\ \mbox{$\pi^*(y|x)>0$} 
only if $\sigma_x(y) = |\mathcal{Y}_x|$.
Hence \mbox{$\mathcal{R}_{\sigma_x}(\pi^*|x) = |\mathcal{Y}_x|$}.

%% file: proof.tex
\section{Proof of Convergence for \rise{}}
\label{sec:proof}
%
%\rise{} involves two function approximation tasks:
%the L2R-based approximation of the ranking function
%and RL-based optimal policy approximation.
%Both approximations may introduce some errors.

In this section, we prove that if the errors in both
L2R and RL are bounded, the error of 
the final output of \rise{} is a linear polynomial of the
errors from the two steps.
First, we define \emph{convergent} L2R,
which can learn an approximated ranking $\Hat{\sigma}_x^U$
``close'' to the ground-truth ranking $\ChM[\sigma_x]{\sigma_x^*}$. 
\begin{defi}
\label{def:pl_convergent}
%Let $\mathcal{X}$ be the set of all possible inputs.
%For any $x \in \mathcal{X}$,
%let $\sigma_x$ be the ground-truth ranking on $\mathcal{Y}_x$.
Let $U$ be the utility function learnt
by an L2R algorithm. For $x \in \mathcal{X}$,
let $\Hat{\sigma}^U_x$ be the ranking derived 
from $U$ on $\mathcal{Y}_x$.
The L2R algorithm is 
\emph{$\epsilon$-convergent} with respect to $\ChM[\sigma_x]{\sigma_x^*}$ if
there exists 
a non-negative integer $\epsilon$
%\todo{YG: I change the definition of epsilon here, because epsilon
%can only integers, so as the sigma function.
%MM: No, the difference between sigmas is a natural number but epsilon can be any positive real number. we can keep it real as it is more popular for errors.
%-YG: right, epsilon can be real number, but then the bound
%will not be tight
%--ChM: tightness of the bound is a good argument, so I would vote for $\mathbb{N}$}
such that %\ChM[$\forall x \in \mathcal{X}$]{for any $x \in \mathcal{X}$}, 
\begin{align}
%\lim \limits_{n \to \infty} 
\max_{y\in\mathcal{Y}_x} 
|\hat{\sigma}^U_x(y) - \ChM[\sigma_x]{\sigma_x^*}(y)| \leq \epsilon.
\label{eq:pl_convergent}
\end{align}
\end{defi}
%\ChMN{replaced $s$ with $g$ for consistency with other formulas.}

Some L2R algorithms are $\epsilon$-convergent
under realistic conditions.
For example, the quick-sort based L2R algorithm proposed by
\newcite{DBLP:conf/icml/MaystreG17} is $\epsilon$-convergent
with respect to pairwise oracles, when it %the L2R algorithm
obtains a sufficiently large number of preferences.
Next, we define convergent RL algorithms \MM[, 
which]{that} guarantee to learn near-optimal policies. 

\begin{defi}
\label{def:rl_convergent}
Let $M_x(\ChM[\sigma]{\sigma_x})$ denote the episodic MDP 
for \MM[the]{an} \nlg{} task \MM[for]{given} 
input $x \in \mathcal{X}$,
which uses a 
ranking function $\ChM[\sigma]{\sigma_x}$ over $\mathcal{Y}_x$ 
as its reward function.
%\ChM[as]{for} its \ChM[rewards]{reward} function\ChM{\ $\mathcal{R}_{\sigma_d}$}. 
An RL algorithm is \MM{\mbox{\emph{$\delta$-convergent}}} for $M_x(\ChM[\sigma]{\sigma_x})$ if
there exists 
\MM[$\delta \in \mathbb{R}, \delta \geq 0$ ]{$\delta \in \mathbb{R}_{\geq 0}$}
such that
%for any $x \in \mathcal{X}$,
\begin{align}
\lim \limits_{n \to \infty} %\max \limits_{g\in\mathcal{S}_d}
\mathcal{R}_{\ChM[\sigma]{\sigma_x}}(\pi^*|x) - \mathcal{R}_{\ChM[\sigma]{\sigma_x}}(\pi_n|x)  \leq \delta,
\label{eq:rl_convergent}
\end{align}
where $\pi^*$ is the optimal policy for $M_x(\ChM[\sigma]{\sigma_x})$
and $\pi_n$ is the policy of the RL algorithm
after $n$ episodes of learning.
\footnote{Since $\pi^*$ is the optimal policy for $\ChM[\sigma]{\sigma_x}$, 
$\mathcal{R}_{\ChM[\sigma]{\sigma_x}}(\pi^*|x)$ is always higher than the expected reward of any other policy (cf. Eq.~\ref{eq:rl_obj}).
}
\end{defi}
%\MMN{Shouldn't it be absolute value? -- decision: we add a footnote to clarify why no abs --\textbf{ChM: added a footnote, please check}}
\MM[Multiple]{Many} RL algorithms have been proven to be \MM[$\delta$-convergent]{\mbox{$\delta$-convergent}},
including value-based RL \cite{sutton2009fast}, 
actor-critic \MM[algorithms \cite{DBLP:conf/nips/SuttonMSM99}]{\cite{DBLP:conf/nips/SuttonMSM99}} and policy \MM[gradient algorithms]{gradient} %\ChM[as such]{such as} REINFORCE
\cite{DBLP:conf/icml/Fazel0KM18}.

Recall that for an input $x$ from the test 
set $\mathcal{X} \setminus \overline{\mathcal{X}}$,
our \ChM[target]{goal} is to find the optimal policy $\pi^*$
for the \nlg{} tuple $(\ChM[\sigma_x]{\sigma_x^*}, M_x(\ChM[\sigma_x]{\sigma_x^*}))$
(see \S\ref{sec:background}). %subsec:back:emds
The theorem below shows that if \rise{} uses 
convergent L2R and RL, its output is near-optimal,
i.e. the error between the \rise{} output and 
the optimal output is bounded.
%a convergent L2R algorithm to learn the 
%approximated ranking function $\hat{\sigma}_x^{\ChM{U}}$ and a convergent RL algorithm to obtain the optimal policy
%$\Hat{\pi}$ for the MDP $M_x(\hat{\sigma}_x^{\ChM{U}})$,
%$\Hat{\pi}$ is a near-optimal policy for $M_x(\ChM[\sigma_x]{\sigma_x^*})$, 
%i.e.\ its error to $\pi^*$ is bounded.

\begin{theorem}
\label{the_theorem}
%Let $D$ be a set of possible inputs. 
For a test input $x \in \mathcal{X} \setminus \overline{\mathcal{X}}$, 
we denote the optimal \nlg{} agent for $x$ by $(\sigma^*_x, M_x(\sigma^*_x))$,
and its \rise{} approximation
by $(\hat{\sigma}_x^{\ChM{U}}, M_x(\hat{\sigma}_x^{\ChM{U}}))$,
where $\Hat{\sigma}_x^{\ChM{U}}$ is the approximation of $\sigma^*_x$
learnt by an $\epsilon$-convergent L2R algorithm.
Let $\pi^*$ be the optimal policy for $M_x(\sigma^*_x)$
and $\Hat{\pi}_n$ the proper policy for $M_x(\hat{\sigma}_x^{\ChM{U}})$
learnt by a $\delta$-convergent RL algorithm for $n$ episodes.
Then, as $n$ approaches positive infinity, 
$\mathcal{R}_{\sigma^*_x}(\hat{\pi}_n|x) \geq 
\mathcal{R}_{\sigma^*_x}(\pi^*|x) - \delta - \epsilon$.
\end{theorem}

\begin{proof}
We denote the optimal policy for $M_x(\hat{\sigma}_x^{\ChM{U}})$
as $\Hat{\pi}^*$.
Hence from Eq.~\eqref{eq:rl_convergent} we have: 
%\todo{YG: most Eq.s are actually inequations. is that ok? ChM: I consider Eq. as a somewhat generic caption. I just added it to make sure readers understand that we're referring to some formula. Without Eq., it's generally fine as well, but in some contexts, it can be a bit difficult, e.g., when refering to (1)...}:
\begin{align}
\lim_{n \rightarrow \infty} 
\mathcal{R}_{\Hat{\sigma}_x^U}(\hat{\pi}_n|x) \geq 
\mathcal{R}_{\Hat{\sigma}_x^U}(\Hat{\pi}^*|x) - \delta.
\label{eq:step1}
\end{align}
For any $y \in \mathcal{Y}_x$, 
from Eq.~\eqref{eq:pl_convergent} we have
\begin{align}
%\sigma_d(g) - \epsilon \leq 
\MM[
\Hat{\sigma}_x^{U}(x) \leq \sigma^*_x(x) + \epsilon]
{\Hat{\sigma}_x^{U}(y) \leq \sigma^*_x(y) + \epsilon}.
\label{eq:extend_epsilon}
\end{align}
Using Eq.~\eqref{eq:extend_epsilon}, we can bound $\mathcal{R}_{\hat{\sigma}_x^{\ChM{U}}}(\Hat{\pi}_n|x)$ with 
any positive $n$ value (see Eq.~\eqref{eq:rl_obj}):
\begin{align}
\mathcal{R}_{\hat{\sigma}_x^{\ChM{U}}}(\hat{\pi}_n|x) & = \sum_{x \in \mathcal{Y}_x}
\Hat{\pi}_n(y|x)\hat{\sigma}_x^{\ChM{U}}(y) \nonumber \\
& \leq \sum_{x \in \mathcal{Y}_x}
\Hat{\pi}_n(y|x)[\sigma^*_x(y)+\epsilon] \nonumber \\
& = \sum_{y \in \mathcal{Y}_x}
\Hat{\pi}_n(y|x)\sigma^*_x(y)+
\epsilon \sum_{y \in \mathcal{Y}_x}
\Hat{\pi}_n(y|x) \nonumber \\
& = \mathcal{R}_{\sigma^*_x}(\hat{\pi}_n|x) + \epsilon.
\label{eq:step2}
\end{align}
%\ChMN{wouldn't the last line yield $ \mathcal{R}_{\sigma_d}(\hat{\pi}_n|d) + \epsilon = \left(\sum_{g \in \mathcal{S}_d} \Hat{\pi}_n(g|d)\sigma_d(g)\right) + \epsilon$, which is different from the previous to last line? Basically the sum in $\mathcal{R}_{\sigma_d}(\hat{\pi}_n|d) + \sum_{g\in \mathcal{S}_d}\Hat{\pi}_n(g|d)\epsilon$ would have to be equal to $\epsilon$ which is not the case unless I am mistaken...}
%
Note that
$\sum_{y \in \mathcal{Y}_x} \Hat{\pi}_n(y|x) = 1$ in Eq.~\eqref{eq:step2}
because $\Hat{\pi}_n$ is a proper policy (see \S\ref{sec:background}). %subsec:back:rl
Combining Eq.~\eqref{eq:step1} and \eqref{eq:step2}, we get %with $n$ approaches positive infinity,
\begin{align}
\lim_{n \rightarrow \infty} \mathcal{R}_{\sigma^*_x}(\Hat{\pi}_n|x) \geq
\mathcal{R}_{\hat{\sigma}_x^{\ChM{U}}}(\Hat{\pi}^*|x) - \delta - \epsilon.
\label{eq:step3}
\end{align}
Since $\hat{\pi}^*$ is the optimal policy for $M_{\hat{\sigma}_x^{\ChM{U}}}$,
according to Eq. \eqref{eq:rl_obj},
%the definition of the ranking function 
%given in \S\ref{sec:background}, %subsec:back:emds
$\mathcal{R}_{\hat{\sigma}_x^{\ChM{U}}}(\Hat{\pi}^*|x) = |\mathcal{Y}_x|$.
Similarly, $\mathcal{R}_{\sigma^*_x}(\pi^*|x) = |\mathcal{Y}_x|$.
Hence we can replace $\mathcal{R}_{\hat{\sigma}_x^{\ChM{U}}}(\Hat{\pi}^*|x)$
in Eq.~\eqref{eq:step3} with 
$\mathcal{R}_{\sigma^*_x}(\pi^*|x)$ and obtain
%\begin{align}
$\lim_{n \rightarrow \infty} \mathcal{R}_{\sigma^*_x}(\Hat{\pi}_n|x) \geq 
\mathcal{R}_{\sigma^*_x}(\pi^*|x) - \delta - \epsilon.
$
%\end{align}
\end{proof}

%% file: experiments.tex
% \begin{table}
%   \centering\small
%     \begin{tabular}{l c  c  c}
%       \toprule
%       Dataset & \#\,\ChM[DocCluster]{Cluster} & \#\,Doc & \#\,Sent/Cluster\\
%       \midrule
%       DUC$\,$'01 &  30 & 308 & 378 \\
%       DUC$\,$'02 &  59 & 567 & 271 \\
%       DUC$\,$'04 &  50 & 500 & 265 \\
%       \bottomrule
%     \end{tabular}
%   \caption{Statistics of the DUC datasets: number of document clusters (\#\,Cluster), total number of documents (\#\,Doc) and average number of sentences per cluster (\#\,Sent/Cluster). 
%   }
%   \label{table:datasets}
% \end{table}

\section{Experimental Setup}
\label{sec:experiments}

\paragraph{Task and datasets.}
We evaluate \MM[multiple variants of \rise{}]{\rise{}} for 
extractive multi-document summarisation
% EMDS is introduced in sec 1
on three benchmark datasets from 
the Document Understanding Conferences 
(DUC)\footnote{\url{https://duc.nist.gov/}}\ChM{\ described in Table~\ref{table:datasets}}. 
Each dataset contains a set of document clusters. 
For each cluster, the target
is to create a summary with at most 100 words.
Each cluster is accompanied with several \MM{\mbox{human-generated}} summaries\ChM{,
denoted as $r_x$,} that \ChM[are used]{we use} for training and evaluation.
%
%Table~\ref{table:datasets} shows some statistics of the datasets.
%
To decide the best parameters, we perform  
\MM{\mbox{10-fold}} cross validation on DUC'01. 
% \MMN{Do you mean hyper-parameters?}

%Due to the relatively small size of these datasets, to the best
%of our knowledge, \ctopic{} RL has not been used
%to train summarisers on these datasets.

\paragraph{Ground-truth reward oracle.}
%\label{subsec:exp:gtranking}
For $x \in \overline{\mathcal{X}}$,  we use 
\ChM[ROUGE to provide]{a ROUGE-based $U^*$ to induce}
%\ChM[metrics]{\cite{lin2004rouge}} 
%the citation of rouge has been given earlier in intro; save one line 
the \mbox{ground-truth} $\sigma^*_{x}$\MM[:]{.} 
\MM[For a summary $y \in \mathcal{Y}_x$, w]
{W}e measure \MM[its quality]{the quality of a summary $y \in \mathcal{Y}_x$}  by 
\MM[]
{
\mbox{$U^*(y, x) = R_1(y,r_x)+2R_2(y,r_x)$,}
}
where $R_1$ and $R_2$ denote the average ROUGE-1 and ROUGE-2 
recall metrics,
%\footnote{Arguments in ROUGE computation:
%\verb|-n 4 -m -a -x -c 95 -r 1000 -f A -p 0.5 -t 0 -2 -4 -u.}
respectively\ChM[, 
and $r_x$ are reference summaries for $x$ provided
by the DUC datasets.]{.} 
$R_1$ and $R_2$ are arguably the most widely used metrics
to approximate human evaluation of summary quality. 
\ChM[Since for optimal summaries in DUC,]{The} $R_2$ scores 
\ChM{for optimal extractive summaries} are usually half of their
$R_1$ scores \MM[(cf.\ \cite{DBLP:conf/emnlp/GaoMG18})]{\cite{DBLP:conf/emnlp/GaoMG18}}\ChM[;]{, so} 
we multiply $R_2$ \MM[with]{by} $2$
to balance the impact of $R_1$ and $R_2$. 
Next, we describe how we approximate the ground-truth reward.
%The ground-truth ranking $\sigma_x$ will be induced from $U^*_x$. -> added this information to the intro sentence

\input{dataset.tex}

%\ChMN{Approximated reward oracle?}
%\paragraph{L2R.}
\subsection{Approximated Reward Oracle}
%\label{subsec:exp:l2r}
Recall that the goal of L2R is to learn the utility function $U$ 
using pointwise or pairwise preferences 
(see \S\ref{sec:background}). %subsec:back:pref_learn
In pointwise L2R, we sample $N=3000$ different summaries
for each document cluster in $x \in \overline{\mathcal{X}}$
\ChM[, \MM{where} all meeting the 
100-word \MM{length} requirement and are generated by randomly selecting
sentences from the input documents.]{ by randomly selecting sentences from $x$ until the length limit is reached.}
Using larger $N$ does not significantly increase
the quality of $\hat{\sigma}_x^U$ \MM[, but]{but} increases the training time.
We use the ground-truth ranking $\sigma^*_x$ over the $|\overline{\mathcal{X}}| \cdot N$ \ChM[sampled summaries]{samples} to learn the utility function $U$ by minimising the MSE loss \ChM[in]{from} Eq.~\eqref{eq:mseloss}.
In pairwise L2R, \ChM[from the $|\overline{\mathcal{X}}| \cdot N$ summaries sampled in pointwise L2R, we can make at most 
$|\overline{\mathcal{X}}| \frac{N \cdot (N - 1)}{2}$ pairs of summaries.
To balance the performance and training time, 
%based on our pilot studies, 
we randomly draw $10^5$ pairs to compute the function $U$]{%
we randomly draw $K=10^5$ of the $|\overline{\mathcal{X}}| \frac{N \cdot (N - 1)}{2}$ \MM[theoretically]{} possible pairs of the 
$|\overline{\mathcal{X}}| \cdot N$ sampled summaries \MM[and use them to]{to} compute $U$} 
by minimising $\mathcal{L}^{\mathrm{CE}}$, %(Eq.~\eqref{eq:ce_loss}),
$\mathcal{L}^{\mathrm{MR}}$ %(Eq.~\eqref{eq:hinge_loss}) 
or 
$\mathcal{L}^{\mathrm{IM}}$ %(Eq.~\eqref{eq:weighed_hinge_loss}).
from \mbox{Eq.~\eqref{eq:ce_loss} -- \eqref{eq:weighed_hinge_loss}}.
Preliminary results suggest that increasing 
%\ChM[sample pair number]{number of sampled pairs}
$K$ to $10^6$ does not benefit the performance but
significantly slows down the training,
while decreasing it to $10^4$ significantly harms the quality of $\hat{\sigma}^U_x$. 
%\MM{The learnt utility function can be used as an approximated reward oracle or reward function (see Fig.~\ref{fig:workflows}).}
%\todo{omit or shorten due to space}

%\paragraph{$U$ Representation.}
%\label{subsec:exp:feature}
%\MMN{should it be really a sub-section? I  cannot see it as a sub-section, unless we change sub-sectioning structure in Section 4. To me this is part of L2R.}
%\ChMN{ChM: maybe promote sec 4.3 to a full sec 5 (`results + discussion') and let sec 4 (title `experimental setup') take no subsecs, but only paragraphs? }

\ChM{We use a linear model to approximate the utility function, i.e.\ $U(y,x) = w \cdot \phi(y,x)$, where $\phi(y,x)$ encodes $y$ for input $x$ as a vector by concatenating the following features:}

\begin{itemize}
\item 
The negative Jensen-Shannon divergence between 
unigrams, bigrams and named entities in $y$ and $x$.
\item 
The summary quality evaluation heuristics 
proposed by \newcite{rioux2014emnlp}.
\item 
TFIDF values averaged over all tokens
and named entities in $y$ \cite{DBLP:conf/naacl/PeyrardG18}.
\item 
The average number of document clusters a word or 
named entity appears in.
\item 
Rate of named entities from $x$ appearing in $y$.
\item 
The percentage of tokens in $y$ that belong to some named entities in $x$.
\item
The redundancy feature proposed by \newcite{DBLP:conf/naacl/PeyrardG18},
applied to unigrams, bigrams and named entities.
\end{itemize}

We additionally use two 
Convolutional Neural Network (CNN) architectures to generate
auto-encoded features: \mbox{StdCNN} \MM[proposed
by \newcite{DBLP:conf/emnlp/Kim14}]{\cite{DBLP:conf/emnlp/Kim14}}
and PriorSum \MM[proposed by \newcite{DBLP:conf/acl/CaoWLLZW15}]{\cite{DBLP:conf/acl/CaoWLLZW15}}.
% \MMN{Not clear!}
To train \MM[the CNN-based]{these} auto-encoders, we \MM{follow \newcite{DBLP:conf/acl/CaoWLLZW15} and}
feed only the embeddings of the words \MM[in]{of} $y$ into 
these models\MM[, and]{. We } add a linear layer as the last layer to map the output vector of these models to a $U$ value. 
% \ChM[Note that, in line with the setup in]{In line with} \newcite{DBLP:conf/acl/CaoWLLZW15},  the input cluster $x$ is not fed into the CNN. 
%
Full settings of \MM[the CNNs]{these models} are in the supplementary  material.
%\ChM{\color{magenta}Note that the linear model loses the convergence property of \newcite{DBLP:conf/icml/MaystreG17}, but speeds up the training significantly without sacrifying the quality (see \S\ref{sec:results}).}
%
For each $x \in \mathcal{X} \setminus \overline{\mathcal{X}}$,
we measure the quality of the \ChM[L2R obtained ranking function]{approximated ranking}
 $\hat{\sigma}^U_x$
by its ranking correlation to the \ChM[ranking]{ground-truth} $\sigma^*_x$%
\MM[
induced from $U^*$.]
{
.
} 
% should do without the reference to U^*, since its now the same section anyway!
We consider two ranking correlation metrics:
Spearman's $\rho \MM{\ \in [-1,1]}$ and the Normalized Discounted Cumulative Gain \MM{(\NDCG{})} 
on the top $1\%$ items\MM[(\NDCG{})].
\NDCG{} \MM{$\in [0,1]$} \ChM[is a metric that compares ranked lists and ]{}puts more emphasis on the top elements with 
logarithmic decay weighting \MM[\cite{DBLP:journals/tois/JarvelinK02}].
\MM[The $\rho$ values are between $-1$ and $1$, while the
\NDCG{} values are between 0 and 1]{}For both metrics, higher values \MM[mean]{indicate} \MM[higher correlation]{stronger correlations} .

%\paragraph{Feature Selection.}
%We select features that yield higher $\rho + \NDCG{}$, 
%averaged over results for all 10 folds of DUC'01. 
%
We train the utility function $U$ with 
the loss functions in Eq.~\eqref{eq:mseloss}-\eqref{eq:weighed_hinge_loss}
and we find that using different loss functions does 
not significantly\footnote{In all experiments, 
we preform statistical significant test by 
\mbox{two-tailed} $t$-test and \ChM[p\_value]{$p$-value} $<0.01$. }
change performance (see supplementary material). 
However, the cross-entropy loss consistently
\MM[yield]{results in} marginally better performance than the others,
and hence we use it throughout our experiments.
%
%Following standard feature selection methods,
%\MM[i.e.]{we} \MM[filtering out]{omit} features that are highly correlated with
%other features or have small weights in the linear model.
%The final features we used \MM[in the reward representation]{to represent} $\phi$ are marked with \ChM{an} asterisk in the \MM{aforementioned} feature list.
%
We find that under %\MM{the} 
all \MM{examined} settings\MM[we have tested], 
the CNN-based features \MM[fail to outperform]{underperform} the other features\MM[, and u]{. U}sing the CNN-based features together with
other features \MM[fail to significantly improve]{also worsen} 
 the quality of $\hat{\sigma}^U_x$.
The reason is because both PriorSum and StdCNN
only encode the summaries' \emph{document-independent features}
\cite{DBLP:conf/acl/CaoWLLZW15}. 
\MM[To encode]{Encoding} \MM[the]{\mbox{document-dependent}} features
%i.e. inputs $x$ and the relation between $y$ and $x$
requires more sophisticated neural models
\MM[, e.g., \newcite{DBLP:conf/aaai/WuH18} or \newcite{DBLP:journals/corr/abs-1802-08636}]
{\cite{DBLP:conf/aaai/WuH18,DBLP:journals/corr/abs-1802-08636}}
which require considerable time, data and
parameter tuning\MM[, and hence]{. This}  undermines the benefits of \rise{}.
We leave the research for efficient reward representation
learning for future work.

\input{reward_quality.tex}

\subsection{\rise{} Setup}
%After feature selection, 
We test \rise{} on all three DUC datasets as follows.
In line with \newcite{DBLP:conf/aaai/CaoLLW17} and \newcite{DBLP:journals/tois/RenCRWNMR18},
we split train and test data in the ``leave-one-out'' \MM[manner,]{manner} 
so that documents from two datasets are used as the training set, 
$\overline{\mathcal{X}}$, and documents from the rest as 
the test set $\mathcal{X} \setminus \overline{\mathcal{X}}$.
%Results presented in this section are all
%averaged over all test sets. 
%
%QUI
%
%We implement \rise{} in PyTorch. 
In each run in the ``leave-one-out'' experiments,
we randomly select 30\% data from the training set as the
dev set, and select the model with the best
performance on the dev set.
We use Adam %as the gradient optimiser 
with \MM{initial} learning rate $10^{-2}$.
The number of epochs is $10$ and batch size is $2$.
As for RL in \rise{},
we use the same \emph{temporal difference} algorithm
as \newcite{rioux2014emnlp}.
Full details of our parameter selection
and results of using different loss functions
are in the supplementary material.
%\footnote{%Supplementary material: 
%\url{https://www.dropbox.com/s/slddgj5yfor5u2g/RELIS_IJCAI19_SupplementaryMaterial.pdf?dl=0}}.

%% file: dataset.tex
\begin{table}[t]
  \centering\small
    \begin{tabular}{l c  c  c}
      \toprule
      Dataset & \#\,\ChM[DocCluster]{Cluster} & \#\,Doc & \#\,Sent/Cluster\\
      \midrule
      DUC$\,$'01 &  30 & 308 & 378 \\
      DUC$\,$'02 &  59 & 567 & 271 \\
      DUC$\,$'04 &  50 & 500 & 265 \\
      \bottomrule
    \end{tabular}
  \caption{\MM{Some} statistics of the DUC datasets: \MM{the} number of document clusters (\#\,Cluster), \MM{the} total number of documents (\#\,Doc) and \MM{the} average number of sentences per cluster (\#\,Sent/Cluster). 
  }
  \label{table:datasets}
\end{table}

%% file: reward_quality.tex
\begin{table}[!t]
    \centering\small
    %\begin{tabular}{l c c | c c | c c }
    \begin{tabular}{l *{3}{ c @{\hspace{3mm}} c }}
    \toprule
      & \multicolumn{2}{c}{DUC'01} & \multicolumn{2}{c}{DUC'02} 
      & \multicolumn{2}{c}{DUC'04} \\
     & $\rho$ & \NDCG  & $\rho$ & \NDCG & $\rho$ & \NDCG \\
     \midrule
     ASRL & .176 & .555 & .131 & .537 & .145 & .558 \\
     REAPER & .316 & .638 & .301 & .639 & .372 & .701\\
     JS & .549 & .736 & .525 & .700 & .570 & .763 \\
     Our $\hat{\sigma}^U_x$  & \textbf{.601} & \textbf{.764} 
     & \textbf{.560} & \textbf{.727} & \textbf{.617} & \textbf{.802} \\
     \bottomrule
    \end{tabular}
    \caption{\ChM{\MM[C]{The} correlation of approximated and ground-truth ranking. $\hat{\sigma}^U_x$ has significantly higher correlation over all other approaches.}}
    \label{table:correlation}
\end{table}

%% file: results.tex
\section{Results}
\label{sec:results}

%\subsection{Reward Quality}
Table~\ref{table:correlation} compares the quality
of our $\hat{\sigma}^U_x$
\MM[(using the features selected in \S\ref{sec:experiments})
and]{with} other widely used rewards for \intopic{} RL (see \S\ref{sec:experiments}).
% \MMN{something that is missing here is explanations of rewards. YG: described in the feature list} 
$\hat{\sigma}^U_x$ has significantly higher correlation
to the ground-truth ranking compared \MM[to]{with} \ChM[other
rewards]{all other approaches,} \MM[.]{confirming that our proposed L2R method yields a superior reward oracle.} %
\MM[We see that the higher quality of $\hat{\sigma}^U_x$ helps \rise{} generate significantly better summaries compared to methods using other rewards.]{} %
%This experimentally shows that even a linear $U$ function trained by a proper L2R model yields superior performance.  

\input{relis_rouge.tex}

%\subsection{Summary Quality}
\ChM[W]{In Table~\ref{table:duc_results}, w}e compare \rise{} \MM[against]{with} \MM[\mbox{RL-based} and \mbox{non-RL-based}]{\mbox{non-RL-based} and \mbox{RL-based}} summarisation systems. %(Table~\ref{table:duc_results}).
%
% I guess, we don't have compare against 
% https://books.google.com/ngrams/graph?content=compare+against%2C+compare+with&year_start=1800&year_end=2000&corpus=15&smoothing=3&share=&direct_url=t1%3B%2Ccompare%20against%3B%2Cc0%3B.t1%3B%2Ccompare%20with%3B%2Cc0
%
\MM[As f]{F}or \MM[non-RL-based]{\mbox{non-RL-based}} systems, we \MM[consider]{report} 
%the unsupervised \emph{LexRank}  \cite{DBLP:journals/jair/ErkanR04}
%that uses graphs to select  the salient sentences based on the PageRank algorithm, 
\emph{ICSI} \cite{Gillick:2009:SGM:1611638.1611640}
\ChM[that uses an integer linear programming solver
to build summaries by maximising the bigram overlap between
the summary and the input documents]{maximising the bigram overlap of summary and input using integer linear programming},
\ChM[the CNN-based ]{}\emph{PriorSum} \cite{DBLP:conf/acl/CaoWLLZW15}
\ChM[that uses CNNs to learn the quality of sentences,]{learning sentence quality with CNNs,}
\emph{TCSum} \cite{DBLP:conf/aaai/CaoLLW17}
\ChM[that first uses text classification methods
to find the category a document cluster falls into, and uses
the category information to improve the summarisation performance]{employing text classification of the input documents},
the variant of \mbox{TCSum} \ChM[that without]{without the} text classification pre-training
(TCSum$^-$) 
%\MMN{(1) TCSum-TxtClf is not clear, maybe TCSum w/o TxtClf },
and \ChM[the \emph{SRSum} system]{\emph{SRSum}} \cite{DBLP:journals/tois/RenCRWNMR18},
which learns sentence relations with both \ChM[both word-level]{word-}
and sentence-level attentive neural networks \ChM[and
uses the learnt sentence relations to
learn the]{to estimate} salience\ChM[ levels of sentences in the input documents]{}. 

%As for RL-based baselines, 
%we implement both \intopic{} and \ctopic{} RL.
%
For RL-based systems, we re-implement REAPER 
\cite{rioux2014emnlp} as an \intopic{} RL, and 
DeepTD as a \mbox{\ctopic{}} RL. 
DeepTD is adapted from the DQN-based RL summariser 
\cite{DBLP:journals/ijon/YaoZLW18}
and is trained by taking $U^*$ as the rewards (see \S\ref{sec:experiments}).
It uses InferSent to represent summaries. 
%
% For \intopic{} RL, we implement REAPER \cite{rioux2014emnlp}\ChM[;]{\ and}
% for \ctopic{} RL\ChM[ (DeepTD), 
% we use $U^*_x$]{, we implement DeepTD using $U^*$} (see \S\ref{sec:experiments}) as the rewards
% \ChM[to train]{for training} an adapted DQN-based \ChM[\ctopic{} RL]{RL} summariser 
% proposed by \newcite{DBLP:journals/ijon/YaoZLW18}
% with InferSent as the summary representation.
%
To improve the efficiency and performance of
DeepTD, we use\ChM[ multiple techniques including]{}
\emph{memory replay} and \emph{periodic update}.
Further details and \MM[the pseudo code]{the algorithm} of DeepTD 
are in the supplementary material.
%we have also tried to train the bandit-based
%\ctopic{} summariser proposed by \newcite{DBLP:conf/acl/RiezlerKU18},
%but its training on one dataset can be finished 
%but its training cannot be finished in \MM{a} reasonable \MMN{very vague?} time.
%

%\todo{remove loss functions from table and describe them briefly in the text}
%As Table~\ref{table:duc_results} shows,
\rise{} significantly outperforms the other RL-based systems. 
%\todo{YG: the significance advantage of rise over other RL is not shown in the table.}
Note that \rise{} and REAPER use the identical RL algorithm
for \intopic{} policy learning;
%the only difference between REAPER and \rise{} 
%is their reward function; 
hence the improvement of \rise{}
is due to the higher quality of the L2R-learnt reward $\hat{\sigma}^U_x$.
\rise{} outperforms DeepTD because \ChM[the DUC datasets are too small
to train a \mbox{\ctopic{}} policy indicating that our model does not need very data]{training \mbox{\ctopic{}} policies requires much more data than available in the DUC datasets. 
%\rise{} is a viable solution, as it effectively adapts to the input.
}
\ChM[\rise{} variants also perform]{At the same time, \rise{} performs} on par with neural-based TCSum and SRSum,
while it requires significantly less data \MM[or]{and} time to train, as shown next.
% \MM[We can see that using d]{Surprisingly, d}ifferent loss functions in \rise{} \MM[does]{do} not lead to significant \MM[performance difference]{different performance}
% \MMN{why? Are we surprised?}.
% \YGN{A bit surprised because we expect pairwise L2R to give better performance.
% But it is hard to say why pointwise and pairwise perform similarly,
% because in pairwise, we also have the pair sampling process.
% As we have discussed before, we have tried several sampling strategies
% but they neither result in significant differences; but theoretically
% sampling indeed influences the performance.}
% \MMN{clear! updated the previous sentence.}

%\subsection{Time and Resources}
%\paragraph{Training \YG[time]{expenses}.}
%\todo{ChM: don't like expenses much. Time/resources/effort//scalability/Computation time?}
We run %train and test
\ChM[\rise-CE]{\rise{}}, SRSum, DeepTD and 
REAPER on the same workstation with a 4-core CPU, 
8\,GB memory and no \MM[GPU]{GPUs}. 
Table~\ref{table:time_compare} compares their average
training and test time for each document cluster.
%In each fold, SRSum is trained for 50 epochs as suggested \MM[by the original authors]{in its original settings}. 
% MM: I just removed it to see if anyone sees tha gap. I think we don't need to say this here.
\rise{} reduces the training time of SRSum by two orders of
magnitude\ChM[; a]{. A}t test time, \rise{} takes reasonable time
to train the \mbox{\intopic{}} policy, and we believe \ChM{that} it can be 
further reduced by using more efficient RL algorithms and 
employing techniques like memory replay or reward shaping\ChM[;
we leave it for future work.]{.}
\iffalse
\todo{
\MMN{Why don't we compare the test time, which is more important than training time?}
\YGN{We discuss it here. We deliberately try to avoid discussing it
from the end of the section, because we do not want to conclude
the section with a weakness of our method.
The test time is inevitable a weakness of RELISH. Any ideas about
how to present so as to ``soften'' this weakness?}
\MMN{
(1) 100\% \textbf{agree that the section should not end with this result}; maybe bringing the second paragraph of conclusion here as ChM proposed.
(2) \textbf{suggestion for softening this:} 
    (2.1) the nature of inpSpc models makes it inevitable to have higher test time compared with cross-input methods, as we can compare InpSpc with SRSum and CrsInp-TD
    (2.2) while RELISH test time is almost similar to InpSpc-TD time,  RELISH training time decreases to zero compared with infinite training time for InpSpc-TD
    (2.3) recall that we using the benefits of two paradigms by closing their gap, resulting reducing the training time, keeping the test time fix, and improving the performance 
}
}
\fi

\MM[Compared with]{Unlike} TCSum, \rise{} requires \ChM[far less  data to train]{no additional training data}:
%\todo{no extra data rather than far less}
TCSum uses 1.8 million news articles from New York Times
and their category annotations (e.g.\ Health, Politics, Business) \MM[to train their model]{for training. }%, 
%\MMN{while \rise{} uses .... articles.}
\MM[and they show that]{ It is worth noting that} without \MM{using such a massive extra data for} the text classification \MM[in pre-training]{step}, the performance of TCSum 
significantly drops (see TCSum$^-$ in Table~\ref{table:duc_results}).
\MM[Also, t]{T}he training time of TCSum \MM[can]{is} unlikely \MM[be]{to be} shorter
%\todo{ChM: do we need the information from the footnote?}
than \rise{} \MM[, because]{since} %\footnote{We fail to re-implement TCSum producing 
%the same performance, or obtain its source
%code from the authors.}, 
TCSum requires \MM[pre-training]{to train} 
a CNN-based text classifier \MM[and then]{before} training a CNN-based sentence selector. 
%\todo{ChM: new line break here? (it's summarising the entire section) }

\input{time_comparison.tex}
% \begin{table}%[!ht]
%     \centering\small
%     \begin{tabular}{l r r r r}
%     \toprule
%     & SRSum & DeepTD & REAPER & \rise{} \\
%     \midrule
%     Training & 810\,s & 1,560\,s & N/A & 2\,s  \\
%     Test     & 7\,s & 4\,s & 31\,s & 34\,s \\
%     \bottomrule
%     \end{tabular}
% \iffalse    
%     \begin{tabular}{l r r}
%     \toprule
%     & Training & Test \\
%     \midrule
%     SRSum & 810\,s & 7\,s \\
%     DeepTD & 1560\,s & 4\,s\\
%     REAPER & N/A & 31\,s \\
%     \rise{} & 2\,s & 34\,s \\
%     \bottomrule
%     \end{tabular}
% \fi
%     \caption{Averaged training and test time for each document cluster.
%     Note that REAPER does not have \ChM[the]{a} training phase.}
%     \label{table:time_compare}
% \end{table}

To summarise, \rise{} significantly outperforms \mbox{RL-based} models, and it yields competitive performance compared
\MM[to]{with} the \mbox{state-of-the-art} neural \MM[based methods]{summarisers}, \MM[while requires
only a small fraction of the time or data in training.]
{with the benefit of needing less training time and data.}

%% file: relis_rouge.tex
\begin{table}[!t]
  \centering\small
    \begin{tabular}{l c c | c c | c c}
      \toprule
      & \multicolumn{2}{c|}{DUC'01} & \multicolumn{2}{c|}{DUC'02} 
      & \multicolumn{2}{c}{DUC'04} \\
      %& R1 & R2 & R1 & R2 & R1 & R2 \\
      & $R_1$ & $R_2$ & $R_1$ & $R_2$ & $R_1$ & $R_2$ \\
      \midrule
%      LexRank & 33.43 & 6.09 & 35.29 & 7.54 & 37.87 & 8.88 \\
      %REG-Manual & 34.55 & 7.18 & 34.81 & 8.12 & 37.05 & 9.34 \\ 
      ICSI & 33.31 & 7.33 & 35.04 & 8.51 & 37.31 & 9.36 \\
      PriorSum & 35.98 & 7.89 & 36.63 & 8.97 & 38.91 & 10.07 \\ 
      TCSum & \textbf{36.45} & 7.66 & 36.90 & 8.61 & 38.27 & 9.66 \\ 
      TCSum$^-$  & 33.45  & 6.07 & 34.02 & 7.39 & 35.66 & 8.66 \\ 
      SRSum & 36.04 & 8.44 & \textbf{38.93} & \textbf{10.29} & 39.29 & 10.70 \\
      \midrule
      DeepTD  & 28.74 & 5.95 & 31.63 & 7.09 & 33.57 & 7.96 \\
      REAPER & 32.43 & 6.84 & 35.03 & 8.11 & 37.22 & 8.64 \\
      %REAPER-RIF & 32.34 & 6.81 & 35.02 & 7.73 & 37.20 & 9.31 \\
      \midrule
      %\rise-MSE & 34.90 & 8.40 & 36.67 & 8.89 & 38.63 & 10.15 \\
      %\rise-CE 
      \rise{}
      & 34.73 & \textbf{8.66} & 37.11 & 9.12 & \textbf{39.34} & \textbf{10.73} \\
      %\rise-MR & 34.89 & 8.33 & 37.06 & 9.08 & 39.24 & 10.46 \\
      %\rise-IM & 34.87 & 8.30 & 36.97 & 9.01 &  39.15 & 10.22\\
      \bottomrule
    \end{tabular}
  %\caption{Results of non-RL (top) and RL approaches (centre) compared with \rise{} (bottom). [IS]: \intopic{} RL. [CI]: \ctopic{} RL. 
  %on the DUC datasets. 
  %}
  \caption{Results of non-RL (top)\ChM, \ctopic{} (DeepTD) and \intopic{} (REAPER) RL approaches (\MM[centre]{middle}) compared with \rise{}. %(bottom). 
  %[IS]: \intopic{},  [CI]: \ctopic{}.
  }
  \label{table:duc_results}
\end{table}

%% file: time_comparison.tex
\begin{table}[t]
    \centering\small
    \begin{tabular}{l r r r r}
    \toprule
    & SRSum & DeepTD & REAPER & \rise{} \\
    \midrule
    Training & 810\,s & 1,560\,s & N/A & 2\,s  \\
    Test     & 7\,s & 4\,s & 31\,s & 34\,s \\
    \bottomrule
    \end{tabular}
\iffalse    
    \begin{tabular}{l r r}
    \toprule
    & Training & Test \\
    \midrule
    SRSum & 810\,s & 7\,s \\
    DeepTD & 1560\,s & 4\,s\\
    REAPER & N/A & 31\,s \\
    \rise{} & 2\,s & 34\,s \\
    \bottomrule
    \end{tabular}
\fi
    \caption{Averaged training and test time for each document cluster.
    Note that REAPER does not have \ChM[the]{a} training phase.}
    \label{table:time_compare}
\end{table}

%% file: related_work.tex
%\vspace*{-0.5cm}
\section{Discussion \& Related Work}
\label{sec:related_work}

\rise{} is proposed as a RL-based summarisation paradigm,
but it can be applied to other NLG tasks where
RL is widely used,
e.g., translation, sentence simplification and 
dialogue generation.
For example, to apply \rise{} to translation, 
we just let $\mathcal{X}$ (see \S\ref{sec:background})
be the set of texts in the source 
language and let $\mathcal{Y}_x$
be the set of all possible translations in the
target language for input $x$;
and the error bound of \rise{} (Theorem~\ref{the_theorem})
still holds as it is indifferent to the contents 
in $\mathcal{X}$ and $\mathcal{Y}_x$.
Hence, in this section, we discuss the related
work in the context of NLG in general.

%The strong performance of \rise{} sheds light on the importance
%of \emph{reward learning} for RL, %-based \nlg{} approaches,
\emph{Reward learning} 
recently receives increasing interests from the 
machine learning community 
\cite{DBLP:conf/nips/IbarzLPILA18,DBLP:conf/nips/ZhengOS18},
but it is largely overlooked in \nlg{} until now.
Unlike classic RL applications, \MM[e.g.,]{e.g.}\ robotics and games,
where rewards are provided or %mostly obvious and hence 
easy to design,  
\nlg{} tasks lack strong metrics \MM[for measuring]{to measure}
the quality of the output.
\MM[Widely used]{\mbox{Well-known}} rewards, e.g. \MM[such as BLEU and] ROUGE, \MM[have
been]{are} criticised for their low correlation
with human evaluations \cite{chaganty2018price}. 
Recent work \MM[also] suggests that \MM[a small improvement of]{improving}
the reward function \MM[can significantly boost]{boosts} the performance
of RL-based NLG \cite{DBLP:conf/emnlp/KryscinskiPXS18}.
%
%
%

%\paragraph{Reward design for RL-based NLG.}
Besides using metrics such as BLEU and ROUGE as rewards
to train RL-based NLG,
novel rewards have been designed.
%
%\MM[
%For document summarisation,
%\newcite{DBLP:conf/acl/ArumaeL18}
%convert reference summaries to a set of cloze-style comprehension questions, and propose a question-focused reward function to promote concise, fluent, and informative summaries.
%] % this paper is in the student research workshop and is not worth be here!
%
\newcite{DBLP:conf/emnlp/KryscinskiPXS18} propose a
new reward function for encouraging RL-based summarisers
to \ChM[use more]{prefer} \emph{novel} words in abstractive summaries.
For sentence simplification,
\newcite{DBLP:conf/emnlp/ZhangL17} propose a reward function
measuring the simplicity, relevance and \ChM[grammar]{grammatical} correctness
of candidate outputs.
For machine translation, 
\newcite{DBLP:conf/emnlp/NguyenDB17} propose a simulated
user, which simulates human's ratings on translations,
and they use the simulated user to provide rewards for \ChM{an} RL-based
translator.
However, all rewards discussed above require reference outputs,
unlike our learnt reward function which can be used
to provide rewards at test time when no reference outputs
are available.

\newcite{DBLP:conf/aaai/WuH18} and
\newcite{DBLP:conf/naacl/BosselutCHGHC18} 
recently propose to use a large volume of unlabelled data
to learn a scorer measuring the \emph{
discourse coherence degree} of sentences, and use the scorer
as rewards to train \ctopic{} RL. 
We go beyond their work by proving the 
soundness of the combination of reward learning and RL,
and training \intopic{} instead of  \ctopic{} \ChM[RL]{policies} so as to reduce 
training \ChM[expenses]{time and data}.

\ChM[Some RL-based]{RL-based} \emph{interactive \nlg} methods 
elicit human feedback as rewards.
\newcite{DBLP:conf/acl/RiezlerKU18} elicit 
pointwise and pairwise feedback on candidate
translations to train a \ctopic{} policy.
\newcite{DBLP:conf/emnlp/GaoMG18} propose \ChM[to use
active learning algorithms]{using active learning} to select appropriate candidate
summary pairs \ChM[to human, and use the collected]{and acquire human} preferences
to improve \ChM[the]{a} heuristics-based reward\ChM[s of 
\newcite{ryang2012emnlp}.]{.}
However, these methods require \ChM[considerable human]{much} feedback
to bring satisfactory results\ChM[: for example,
the method proposed by \newcite{DBLP:conf/emnlp/GaoMG18}
needs the user to provide preferences on]{, e.g.,} at least 50 \ChM[pairs
of summaries]{summary pairs} to yield significant improvement\ChM{s \cite{DBLP:conf/emnlp/GaoMG18}}.
\rise{} can be used as a pre-training stage \ChM[for
the above]{for} interactive \ChM[approaches]{methods},
\ChM[such that people can first use \rise{} to 
learn]{as \rise{} first learns} a high-quality \ctopic{} reward function
and then use\ChM{s} the interactive NLG techniques
to elicit a small \ChM[number]{amount} of feedback to fine-tune
a user-\ChM[specific and]{\ and} input-specific reward function,
so at to generate higher-quality and \ChM[more personalised texts]{personalised results}.

%% file: ijcai19.bbl
\begin{thebibliography}{}

\bibitem[\protect\citeauthoryear{Abbeel and
  Ng}{2004}]{abbeel_ng2004app_learning}
Pieter Abbeel and Andrew~Y. Ng.
\newblock Apprenticeship learning via inverse reinforcement learning.
\newblock In {\em ICML}, 2004.

\bibitem[\protect\citeauthoryear{Agarwal and
  Collins}{2010}]{DBLP:conf/ecir/AgarwalC10}
Shivani Agarwal and Michael Collins.
\newblock Maximum margin ranking algorithms for information retrieval.
\newblock In {\em ECIR}, pages 332--343, 2010.

\bibitem[\protect\citeauthoryear{Bosselut \bgroup \em et al.\egroup
  }{2018}]{DBLP:conf/naacl/BosselutCHGHC18}
Antoine Bosselut, Asli {\c{C}}elikyilmaz, Xiaodong He, Jianfeng Gao, Po{-}Sen
  Huang, and Yejin Choi.
\newblock Discourse-aware neural rewards for coherent text generation.
\newblock In {\em NAACL-HLT}, pages 173--184, 2018.

\bibitem[\protect\citeauthoryear{Cao \bgroup \em et al.\egroup
  }{2015}]{DBLP:conf/acl/CaoWLLZW15}
Ziqiang Cao, Furu Wei, Sujian Li, Wenjie Li, Ming Zhou, and Houfeng Wang.
\newblock Learning summary prior representation for extractive summarization.
\newblock In {\em ACL}, pages 829--833, 2015.

\bibitem[\protect\citeauthoryear{Cao \bgroup \em et al.\egroup
  }{2017}]{DBLP:conf/aaai/CaoLLW17}
Ziqiang Cao, Wenjie Li, Sujian Li, and Furu Wei.
\newblock Improving multi-document summarization via text classification.
\newblock In {\em AAAI}, pages 3053--3059, 2017.

\bibitem[\protect\citeauthoryear{Chaganty \bgroup \em et al.\egroup
  }{2018}]{chaganty2018price}
Arun~Tejasvi Chaganty, Stephen Mussmann, and Percy Liang.
\newblock The price of debiasing automatic metrics in natural language
  evalaution.
\newblock In {\em ACL}, pages 643--653, 2018.

\bibitem[\protect\citeauthoryear{Cheng \bgroup \em et al.\egroup
  }{2018}]{DBLP:conf/uai/ChengYWB18}
Ching{-}An Cheng, Xinyan Yan, Nolan Wagener, and Byron Boots.
\newblock Fast policy learning through imitation and reinforcement.
\newblock In {\em UAI}, pages 845--855, 2018.

\bibitem[\protect\citeauthoryear{Fazel \bgroup \em et al.\egroup
  }{2018}]{DBLP:conf/icml/Fazel0KM18}
Maryam Fazel, Rong Ge, Sham Kakade, and Mehran Mesbahi.
\newblock Global convergence of policy gradient methods for the linear
  quadratic regulator.
\newblock In {\em ICML}, pages 1466--1475, 2018.

\bibitem[\protect\citeauthoryear{Gao \bgroup \em et al.\egroup
  }{2018}]{DBLP:conf/emnlp/GaoMG18}
Yang Gao, Christian~M. Meyer, and Iryna Gurevych.
\newblock {APRIL:} {I}nteractively learning to summarise by combining active
  preference learning and reinforcement learning.
\newblock In {\em EMNLP}, pages 4120--4130, 2018.

\bibitem[\protect\citeauthoryear{Gillick and
  Favre}{2009}]{Gillick:2009:SGM:1611638.1611640}
Dan Gillick and Benoit Favre.
\newblock A scalable global model for summarization.
\newblock In {\em ILP}, pages 10--18. Association for Computational
  Linguistics, 2009.

\bibitem[\protect\citeauthoryear{Ibarz \bgroup \em et al.\egroup
  }{2018}]{DBLP:conf/nips/IbarzLPILA18}
Borja Ibarz, Jan Leike, Tobias Pohlen, Geoffrey Irving, Shane Legg, and Dario
  Amodei.
\newblock Reward learning from human preferences and demonstrations in atari.
\newblock In {\em NeurIPS}, pages 8022--8034, 2018.

\bibitem[\protect\citeauthoryear{Kim}{2014}]{DBLP:conf/emnlp/Kim14}
Yoon Kim.
\newblock Convolutional neural networks for sentence classification.
\newblock In {\em EMNLP}, pages 1746--1751, 2014.

\bibitem[\protect\citeauthoryear{Kreutzer \bgroup \em et al.\egroup
  }{2018}]{DBLP:conf/acl/RiezlerKU18}
Julia Kreutzer, Joshua Uyheng, and Stefan Riezler.
\newblock Reliability and learnability of human bandit feedback for
  sequence-to-sequence reinforcement learning.
\newblock In {\em ACL}, pages 1777--1788, 2018.

\bibitem[\protect\citeauthoryear{Kryscinski \bgroup \em et al.\egroup
  }{2018}]{DBLP:conf/emnlp/KryscinskiPXS18}
Wojciech Kryscinski, Romain Paulus, Caiming Xiong, and Richard Socher.
\newblock Improving abstraction in text summarization.
\newblock In {\em EMNLP}, pages 1808--1817, 2018.

\bibitem[\protect\citeauthoryear{Li}{2011}]{DBLP:journals/ieicet/Li11}
Hang Li.
\newblock A short introduction to learning to rank.
\newblock {\em {IEICE} Transactions}, 94-D(10):1854--1862, 2011.

\bibitem[\protect\citeauthoryear{Lin}{2004}]{lin2004rouge}
Chin-Yew Lin.
\newblock {ROUGE}: {A} package for automatic evaluation of summaries.
\newblock In {\em ACL Workshop ``Text Summarization Branches Out''}, pages
  74--81, 2004.

\bibitem[\protect\citeauthoryear{Maystre and
  Grossglauser}{2017}]{DBLP:conf/icml/MaystreG17}
Lucas Maystre and Matthias Grossglauser.
\newblock Just sort it! {A} simple and effective approach to active preference
  learning.
\newblock In {\em ICML}, pages 2344--2353, 2017.

\bibitem[\protect\citeauthoryear{Mnih \bgroup \em et al.\egroup
  }{2015}]{mnih2015human}
Volodymyr Mnih, Koray Kavukcuoglu, David Silver, Andrei~A. Rusu, Joel Veness,
  Marc~G. Bellemare, Alex Graves, Martin~A. Riedmiller, Andreas Fidjeland,
  Georg Ostrovski, Stig Petersen, Charles Beattie, Amir Sadik, Ioannis
  Antonoglou, Helen King, Dharshan Kumaran, Daan Wierstra, Shane Legg, and
  Demis Hassabis.
\newblock Human-level control through deep reinforcement learning.
\newblock {\em Nature}, 518(7540):529--533, 2015.

\bibitem[\protect\citeauthoryear{Narayan \bgroup \em et al.\egroup
  }{2018}]{DBLP:journals/corr/abs-1802-08636}
Shisha Narayan, Shay~B. Cohen, and Mirella Lapata.
\newblock Ranking sentences for extractive summarization with reinforcement
  learning.
\newblock In {\em NAACL-HLT}, pages 1747--1759, 2018.

\bibitem[\protect\citeauthoryear{Nguyen \bgroup \em et al.\egroup
  }{2017}]{DBLP:conf/emnlp/NguyenDB17}
Khanh Nguyen, Hal~Daum{\'{e}} III, and Jordan~L. Boyd{-}Graber.
\newblock Reinforcement learning for bandit neural machine translation with
  simulated human feedback.
\newblock In {\em EMNLP}, pages 1465--1475, 2017.

\bibitem[\protect\citeauthoryear{Peyrard and
  Eckle{-}Kohler}{2017}]{DBLP:conf/acl/PeyrardE17a}
Maxime Peyrard and Judith Eckle{-}Kohler.
\newblock A principled framework for evaluating summarizers: Comparing models
  of summary quality against human judgments.
\newblock In {\em ACL}, pages 26--31, 2017.

\bibitem[\protect\citeauthoryear{Peyrard and
  Gurevych}{2018}]{DBLP:conf/naacl/PeyrardG18}
Maxime Peyrard and Iryna Gurevych.
\newblock Objective function learning to match human judgements for
  optimization-based summarization.
\newblock In {\em NAACL-HLT}, pages 654--660, 2018.

\bibitem[\protect\citeauthoryear{Ranzato \bgroup \em et al.\egroup
  }{2016}]{DBLP:journals/corr/RanzatoCAZ15}
Marc'Aurelio Ranzato, Sumit Chopra, Michael Auli, and Wojciech Zaremba.
\newblock Sequence level training with recurrent neural networks.
\newblock In {\em ICLR}, 2016.

\bibitem[\protect\citeauthoryear{Ren \bgroup \em et al.\egroup
  }{2018}]{DBLP:journals/tois/RenCRWNMR18}
Pengjie Ren, Zhumin Chen, Zhaochun Ren, Furu Wei, Liqiang Nie, Jun Ma, and
  Maarten de~Rijke.
\newblock Sentence relations for extractive summarization with deep neural
  networks.
\newblock {\em {ACM} Trans. Inf. Syst.}, 36(4):39:1--39:32, 2018.

\bibitem[\protect\citeauthoryear{Rioux \bgroup \em et al.\egroup
  }{2014}]{rioux2014emnlp}
Cody Rioux, Sadid~A. Hasan, and Yllias Chali.
\newblock Fear the {REAPER}: {A} system for automatic \mbox{multi-document}
  summarization with reinforcement learning.
\newblock In {\em EMNLP}, pages 681--690, 2014.

\bibitem[\protect\citeauthoryear{Ryang and Abekawa}{2012}]{ryang2012emnlp}
Seonggi Ryang and Takeshi Abekawa.
\newblock Framework of automatic text summarization using reinforcement
  learning.
\newblock In {\em EMNLP/CoNLL}, pages 256--265, 2012.

\bibitem[\protect\citeauthoryear{Sutton \bgroup \em et al.\egroup
  }{1999}]{DBLP:conf/nips/SuttonMSM99}
Richard~S. Sutton, David~A. McAllester, Satinder~P. Singh, and Yishay Mansour.
\newblock Policy gradient methods for reinforcement learning with function
  approximation.
\newblock In {\em NIPS}, pages 1057--1063, 1999.

\bibitem[\protect\citeauthoryear{Sutton \bgroup \em et al.\egroup
  }{2009}]{sutton2009fast}
Richard~S. Sutton, Hamid~Reza Maei, Doina Precup, Shalabh Bhatnagar, David
  Silver, Csaba Szepesv{\'{a}}ri, and Eric Wiewiora.
\newblock Fast gradient-descent methods for temporal-difference learning with
  linear function approximation.
\newblock In {\em ICML}, pages 993--1000, 2009.

\bibitem[\protect\citeauthoryear{Williams}{1992}]{DBLP:journals/ml/Williams92}
Ronald~J. Williams.
\newblock Simple statistical \mbox{gradient-following} algorithms for
  connectionist reinforcement learning.
\newblock {\em Machine Learning}, 8:229--256, 1992.

\bibitem[\protect\citeauthoryear{Wu and Hu}{2018}]{DBLP:conf/aaai/WuH18}
Yuxiang Wu and Baotian Hu.
\newblock Learning to extract coherent summary via deep reinforcement learning.
\newblock In {\em {AAAI}}, pages 5602--5609, 2018.

\bibitem[\protect\citeauthoryear{Yao \bgroup \em et al.\egroup
  }{2018}]{DBLP:journals/ijon/YaoZLW18}
Kaichun Yao, Libo Zhang, Tiejian Luo, and Yanjun Wu.
\newblock {Deep reinforcement learning for extractive document summarization}.
\newblock {\em Neurocomputing}, 284:52--62, 2018.

\bibitem[\protect\citeauthoryear{Zhang and
  Lapata}{2017}]{DBLP:conf/emnlp/ZhangL17}
Xingxing Zhang and Mirella Lapata.
\newblock Sentence simplification with deep reinforcement learning.
\newblock In {\em EMNLP}, pages 584--594, 2017.

\bibitem[\protect\citeauthoryear{Zheng \bgroup \em et al.\egroup
  }{2018}]{DBLP:conf/nips/ZhengOS18}
Zeyu Zheng, Junhyuk Oh, and Satinder Singh.
\newblock On learning intrinsic rewards for policy gradient methods.
\newblock In {\em NeurIPS}, pages 4649--4659, 2018.

\end{thebibliography}
